\documentclass[11pt]{article}

\usepackage[utf8]{inputenc}
\usepackage[T1]{fontenc}
\usepackage{lmodern}
\usepackage{tikz}
\usepackage{pdfpages}
\usetikzlibrary{calc,shapes.geometric,arrows,positioning}
\usepackage{amsmath,amssymb,amsthm,mathtools}
\usepackage{bm}
\usepackage{bbm}
\usepackage{natbib}
\usepackage{geometry}
\usepackage{graphicx}
\usepackage{subcaption}
\usepackage{hyperref}
\usepackage{booktabs}
\hypersetup{
    colorlinks=true,
    linkcolor=blue,
    citecolor=blue,
    urlcolor=blue
}
\usepackage{algorithm}
\usepackage{enumitem}
\usepackage{algorithmic}
\usepackage{amssymb}
\usepackage{latexsym}
\usepackage{tikz}
\usepackage{algorithm}
\usepackage{algorithmic}
\usepackage{subcaption}
\usepackage{amsmath}
\usepackage{amsthm}
\usepackage{bm}
\usepackage{comment}
\usepackage[table]{xcolor}

\newtheorem{lemma}{Lemma}

\theoremstyle{definition}
\newtheorem{definition}{Definition}

\newtheorem{prop}{Proposition}

\begin{document}
\title{Geometric Domain Adaptation via Optimal Transport for Linear Regression in $\mathbb{R}^2$}

\author{
Brian Britos$^{1}$\url{britsimm27@gmail.com}
\and
Mathias Bourel$^{1,2}$ \url{mathias.bourel@fcea.edu.uy}}

\date{}

\maketitle

\begin{center}
$^{1}$ IESTA, Facultad de Ciencias Económicas y de Administración,
Universidad de la República, Montevideo, Uruguay\\
$^{2}$ IRL-2030, Instituto Franco-Uruguayo de Matemática e Interacciones (IFUMI)

\end{center}

\begin{abstract}
Optimal Transport has become recently a powerful method for domain adaptation by aligning source and target distributions. We study a supervised domain adaptation problem where source and target domains are related by a rotation or a translation or a homothety in $\mathbb{R}^2$. We prove that the optimal transport map recovers the underlying map when using a $p-$norm cost with $p \geq 2$. Based on this insight, we develop a method combining $K-$means and optimal transport to estimate the underlying map, enabling adaptation of linear regression models when target data is scarce. Simulations demonstrate improved performance over baseline methods. Rather than relying on highly expressive deep learning architectures, we focus on classical machine learning models to emphasize interpretability and theoretical insight. This perspective allows us to explicitly characterize the role of optimal transport in recovering geometric transformations such as rotations, translations, and homotheties.

Our contributions include a theoretical result linking optimal transport and rotations, translations and homothecies in $\mathbb{R}^2$, and a practical method for adaptation in linear regression offering both conceptual clarity and applied value in domain adaptation tasks in this space.

\end{abstract}

% Preguntar

%% main text
\section{Introduction}
Optimal Transport (OT), introduced by \cite{monge1781} and formalized by \cite{kantorovich1942}, provides a geometric framework for comparing probability measures. Given $\mu \in \mathcal{P}(X)$ and $\nu \in \mathcal{P}(Y)$ and a cost function $c : X \times Y \to \mathbb{R}$, OT seeks the most cost-efficient way to transport $\mu$ to $\nu$. Monge's original formulation looks for a deterministic map $T$ pushing $\mu$ to $\nu$, minimizing $\int_X c(x, T(x)) d\mu(x)$. Since this may lack solutions, Kantorovich relaxed the problem by searching for a coupling $\gamma \in \mathcal{P}(X \times Y)$ with marginals $\mu$ and $\nu$:
\begin{equation*}
    \min_{\gamma \in \Pi(\mu, \nu)} \int_{X \times Y} c(x, y) \, d\gamma(x, y),
\end{equation*}
where $\Pi(\mu, \nu)$ is the set of couplings with marginals $\mu$ and $\nu$, that means $\gamma \in \Pi(\mu, \nu)$ if and only if $\gamma(A \times B) = \mu(A)$ and $\gamma(A \times B)=\nu(B)$ for all measures sets $A \in X$ and $B \in Y$. This relaxation turns OT into a linear program with guaranteed solutions \citep{villani2008optimal}.

Common choices for $c$ include the Manhattan cost $||x - y||_1$, the logistic cost $- \log K(x,y)$ (where $K$ is a kernel), and the widely-used quadratic cost $||x - y||_2^2$, which enjoy plenty of theorical results such as the Brenier theorem \citep{brenier1991polar}. In the discrete setting—relevant to this work—$\mu$ and $\nu$ are empirical distributions over finite sets $\{x_i\}_{i=1}^n$, $\{y_j\}_{j=1}^m$. The OT problem becomes the minimization of $\sum_{i,j} T_{ij} c(x_i, y_j)$ subject to marginal constraints. 
\cite{peyre2019} offer a comprehensive review of this computational OT framework, which has enabled numerous applications in graphics, imaging, and machine learning. A relevant property when $\mu$ and $\nu$ have equal support size (let say $n$) and uniform weights, is that the optimal map is a permutation of the index set $\{1, \ldots, n\}$. For further background, see \cite{villani2008optimal, santambrogio2015optimal, ambrosio2013user}.

\medskip

On the other hand, Domain Adaptation (DA), a subfield of Transfer Learning, aims to transfer knowledge from a labeled source domain $\mathcal{D}^s = \{(x_i^s, y_i^s)\}_{i=1}^{n_s}$ to a target domain $\mathcal{D}^t = \{(x_j^t, y_j^t)\}_{j=1}^{n_t}$, often under distribution shift. As discussed in \cite{mansour2023domain}, a key challenge is the discrepancy between marginal distributions across domains. Classical DA techniques involve instance re-weighting, feature transformation, or regularization. Recently, OT has emerged as a powerful tool to align source and target domains geometrically. For instance, \cite{courty2016optimal} applied OT to unsupervised DA, aligning distributions without target labels—useful in scenarios with calibration issues or sampling bias. For broader perspectives, see \cite{scheck2020learning, david2010impossibility}.

\medskip
To go beyond the general DA setting, in this work we consider a supervised DA problem in $\mathbb{R}^2$ where the source and target domains are related by a geometric transformation: a rotation, translation, or homothety. We show that under mild conditions, optimal transport recovers these transformations exactly when applied to finite sets with equal cardinality. Building on this, we propose a method that combines K-means, optimal transport and simple estimators to infer the transformation and adapt a linear regression model trained on the source domain to the target one. Our simulations demonstrate the advantages of this approach, especially when the target domain has very few labeled samples.

\section{Our method} \label{subsec: R2} 
We propose estimating the rotation angle $\theta \in [0, \pi)$, the translation vector $q \in \mathbb{R}^2$, and the homothety constant $\lambda \in \mathbb{R}$ using optimal transport between discrete empirical measures. This geometric perspective enables us to adapt a linear regression model trained on the source domain so that it performs accurately on the target domain, even when the target domain contains only a small number of labeled samples.

Rather than relying on highly expressive deep learning architectures, we focus on classical machine learning models to better isolate and understand the underlying geometric mechanisms. This choice improves interpretability and provides clearer theoretical insight into how optimal transport captures transformations such as rotations, translations, and homotheties.
For simplicity in what follow, let call movement maps set to the sets of rotation, translations and homothecies, that is
\begin{definition}
We define the set of movement maps $\textbf{MOV}$ as
    \begin{equation*}
        \textbf{MOV} = \Big\{ f: \mathbb{R}^2 \rightarrow \mathbb{R}^2:\, f \text{ is a rotation, translation or homothety} \Big\}.
    \end{equation*}
\end{definition}

We begin with a result concerning how the coefficients of a linear regression changes when applying a function $f \in \textbf{MOV}$.

\begin{lemma} \label{relation_between_parameters}
If $a,b \in \mathbb{R}$ are the coefficients of a line $r$ in the plane $\mathbb{R}^2$ and $\tilde{r}$ is a line obtained from $r$ by applying a rotation of angle $\theta \in [0, \pi)$ or a translation $q=(q_1, q_2) \in \mathbb{R}^2$ or a homothety of factor $\lambda \in \mathbb{R}^+$, then the coefficients of $\tilde{r}$ are

\begin{align*}
    \text{Rotation} & \hspace{0.8 cm} \tilde{a} = \frac{a \cos\theta + \sin\theta}{\cos\theta - a \sin\theta}, \quad \tilde{b} = b (\cos\theta + a \sin\theta), \\
    \text{Translation} & \hspace{2 cm} \tilde{a}=a, \quad \tilde{b}= b + a q_2 - q_1, \\
    \text{Homothety} & \hspace{2 cm} \tilde{a} = a, \quad \tilde{b} = \frac{b}{\lambda}.
\end{align*}
\end{lemma}

\begin{proof}
Let consider a point $(x,y) \in r$ and another pair $(\tilde{x}, \tilde{y}) \in \tilde{r}$, that means $y=ax+b$ and $\tilde{y}=\tilde{a}\tilde{x}+\tilde{b}$.Then, \begin{itemize}\item if $\tilde{r}$ is obtained by applying a rotation then $(\tilde{x}, \tilde{y}) = (x \cos \theta - y \sin \theta, x \sin \theta + y \cos \theta)$, \item  if $\tilde{r}$ is obtained by applying a translation then $(\tilde{x}, \tilde{y}) = (x + q_1, y + q_2)$, \item if $\tilde{r}$ is obtained by applying a homothety then $(\tilde{x}, \tilde{y}) = (\lambda x, \lambda y)$.
\end{itemize}
Replacing in $\tilde{y} = \tilde{a} \tilde{x} + \tilde{b}$ with the previous relation in each case and solving for $\tilde{a}$ and $\tilde{b}$ we obtain the result.
\end{proof}

\subsection{Principal theoretical result}
In the following, we present a result concerning a geometric property of optimal transport in the plane. Later, it will play a critical role in the design of an algorithm for domain adaptation within the same framework. Before this, we need to establish the notion of ``crossing'' relationed with a transport map: 

\begin{definition} We say a transport map $T:X \rightarrow Y$ between two sets in the plane $\mathbb{R}^2$ has a crossing if for any two points in $X$, the segments connecting them to their corresponding points in $Y$ intersect.    
\end{definition}

As an example, the transport map represented by dash black lines in the figure \ref{fig:polygon} has a crossing in the point $O$. We are ready to present the main theoretical result of this work:

\begin{prop} \label{prop_1}
Consider two sets with same size $n$ in $\mathbb{R}^2$\\ $\mathcal{D}^s=\{(x_1^s, y_1^s), \ldots, (x_n^s, y_n^s)\}$ and $\mathcal{D}^t=\{(x_1^t, y_1^t), \ldots, (x_n^t, y_n^t)\}$. If we use uniform measures $\mu$ and $\nu$ over $\mathcal{D}^s$ and $\mathcal{D}^t$ respectively, and assume the cost function is the $p$-norm:
$$
    c((x_i^s, y_i^s), (x_j^t, y_j^t)) = \| (x_j^t, y_j^t) - (x_i^s, y_i^s) \|_p, \quad p \geq 2,
$$
then the optimal transport map $T$ has no crossings.
\end{prop}

\begin{proof}
Since the two sets have the same amount of points and the measures are uniform, the optimal transport map must be a permutation of the index set $\{1, \ldots, n\}$. Formally, given $\sigma \in S_n$ (here $S_n$ is the group of permutations), let $T_{\sigma}$ be the map $(x_i^s, y_i^s) \xrightarrow{T_{\sigma}}(x_{\sigma(i)}^t, y_{\sigma(i)}^t)$. First of all, recall that when the source and target measures are supported on finite sets with the same amount of points and the cost function is convex (our case when using $p \geq 2$), then the Monge problem has a unique solution, as showed for example in \cite{villani2008optimal}. 

To prove the result, we will show that undoing a crossing reduces the total cost, and therefore inductively the optimal transport map cannot have intersections. Let $T$ be a transport map having a crossing at point $O$, represented with back dash lines as in figure \ref{fig:polygon}. We define another transport map $\tilde{T}$ by swapping only the target points of $B$ and $C$ (represented by the red segments in Figure \ref{fig:polygon}). Since $T$ and $\tilde{T}$ differ only on this pair of points, the difference in the total cost is totally explained by the difference in these points, therefore we can focus on their associated costs.

\begin{figure}[h!]
    \begin{center}
        \begin{tikzpicture}
            % Definimos los puntos
            \coordinate (A) at (-0.5,3); % y_j
            \coordinate (B) at (5,1.5); % x_k
            \coordinate (C) at (4,0); % x_i
            \coordinate (D) at (1,0); % y_l

            % Calculamos el punto de intersección de las diagonales AC y BD
            \coordinate (O) at (intersection of A--C and B--D);

            % Dibujamos el polígono
            \filldraw[fill=orange, opacity=0.2, draw=black] (A) -- (B) -- (C) -- (D) -- cycle;
            
            % Dibujamos las diagonales
            \draw[thick, dashed] (A) -- (C);
            \draw[thick, dashed] (B) -- (D);

            % Lados en rojo
            \draw[thick, red] (A) -- (B);
            \draw[thick, red] (C) -- (D);
            
            % Etiquetas de los puntos
            \node[anchor=south] at (A) {$A:(x_j^t, y_j^t)$};
            \node[anchor=west] at (B) {$B:(x_k^s, y_k^s)$};
            \node[anchor=north] at (C) {$C:(x_i^s, y_i^s)$};
            \node[anchor=east] at (D) {$D:(x_l^t, y_l^t)$};

            % Etiqueta del punto de intersección O
            \node[anchor=south] at (O) {$O$};

            \end{tikzpicture}
    \end{center}
    \caption{The dash black segments represent the transport $T$ having a crossing whereas the red segments represent another transport plan $\tilde{T}$ without a crossing.}
    \label{fig:polygon}
\end{figure}
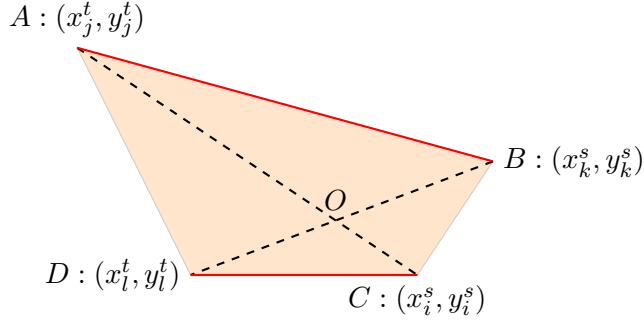

Using the $p$-norm as the cost function implies that the transport cost is the sum of the segment lengths because the integral $\int_X c(x, T(x)) d \mu(x)$ is changed by $\sum_{x \in \text{supp}(\mu)} ||x - T(x)||_p$ where $\text{supp}(\mu) = \{x \in X \hspace{2 mm} | \hspace{2 mm} \mu(x) > 0 \}$ stands for the support of the measure $\mu$. Therefore in order to prove that $\tilde{T}$ reduces the cost with respect to $T$, we show that the sum of the red line lengths is less than that of the black dashed lines. Following the figure \ref{fig:polygon}, lets define $A=(x_j^t, y_j^t), B=(x_k^s, y_k^s), C=(x_i^s, y_i^s)$ and $D=(x_l^t, y_l^t)$. With this notation, the cost of the red segments is $||A-B||_p + ||D-C||_p$ and the cost of the black dash lines is $||A-C||_p + ||D-B||_p$. Finally, by the triangular inequality, valid for any $p$-norm, we have $||A-O||_p + ||O-B||_p > ||A-B||_p$ and $||D-O||_p + ||O-C||_p > ||D-C||_p$, giving:
\begin{align*}
    ||A-C||_p + ||D-B||_p & = ||A-O||_p + ||O-C||_p  \\
    & \hspace{3 mm} + ||D-O||_p + ||O-B||_p \\ 
    & > ||A-B||_p + ||D-C||_p,
\end{align*}
which proves the result.
\end{proof}

The following proposition apply this geometric property of optimal transport in the plane to the case when $f$ is in $\textbf{MOV}$.

\begin{prop} \label{T_equal_f}
    Let $\mathcal{D}^s$ and $\mathcal{D}^t$ be two sets on $\mathbb{R}^2$ of the form $\{(x_1^s, y_1^s), \ldots, (x_n^s, y_n^s)\}$ and \\ $\{(x_1^t, y_1^t), \ldots, (x_n^t, y_n^t)\}$ with the same cardinality. If $\mathcal{D}^s$ and $\mathcal{D}^t$ are related by a map $f \in \textbf{MOV}$, we use uniform measures over $\mathcal{D}^s$ and $\mathcal{D}^t$ and as cost function the $p$-norm, then the optimal transport map $T$ is unique and coincides with the map $f$.
\end{prop}

\begin{proof}
    As we are in the hypothesis of the proposition \ref{prop_1} we know that the optimal transport is unique and also that it has not crossing. Therefore to prove the result, we need to show that the map $f$ is the unique transport map that does not have crossing. 

    In the case when $f$ is a rotation of angle $\theta$, we immediately see that between all possibles transport map (permutation), the rotation is the unique that does not have a crossing and hence has to be the optimal transport map. Similarly, in the case where $f$ is a translation $(x, y) \mapsto (x+q_1, y + q_2)$, then all transport segments are parallel and no crossings occur. Any other permutation would necessarily introduce crossings, so the translation is the unique admissible transport map. Finally, in the case of the homothety $(x,y) \mapsto  (\lambda x,\lambda y)$, the points move radially with respect to the origin, preserving the non-crossing property. Therefore, this matching is again the unique transport map without crossings.

    In all cases, uniqueness implies that the optimal transport map must be exactly $f$.
\end{proof}

As an example, consider the case where the set $\mathcal{D}^s$ is $\{(1,0), (2,0), (3,0)\}$ and $f$ is a rotation of angle $\theta = \frac{\pi}{2}$, that means $\mathcal{D}^t$ is $\{(0,1), (0,2), (0,3)\}$. As the sets have $3$ points, we have $3! = 6$ possible transport plan. The proposition \ref{T_equal_f} assert that only one is optimal and even more that it is the rotation. The $6$ possibles transport map and their associated $|| \cdot ||_2$ costs are show in the figure \ref{fig:example_n_3}.

\begin{figure}[!ht]
    \begin{center}
        \begin{tikzpicture}
            % Transport plan 1
            \coordinate (As) at (0.6, 0);
            \coordinate (Bs) at (1.2, 0);
            \coordinate (Cs) at (1.8, 0);
            \coordinate (At) at (0, 0.6); 
            \coordinate (Bt) at (0, 1.2); 
            \coordinate (Ct) at (0, 1.8); 

            \draw[thick, dashed, gray] (As) -- (At);
            \draw[thick, dashed, gray] (Bs) -- (Bt);
            \draw[thick, dashed, gray] (Cs) -- (Ct);

            \filldraw (As) circle (1 pt);
            \filldraw (Bs) circle (1 pt);
            \filldraw (Cs) circle (1 pt);
            \filldraw (At) circle (1 pt);
            \filldraw (Bt) circle (1 pt);
            \filldraw (Ct) circle (1 pt);

            % Transport plan 2
            \draw[thick, dashed, gray] (As) + (3,0) -- (3, 0.6);
            \draw[thick, dashed, gray] (Bs) + (3,0) -- (3, 1.8);
            \draw[thick, dashed, gray] (Cs) + (3,0) -- (3, 1.2);

            \filldraw (As) + (3,0) circle (1 pt);
            \filldraw (Bs) + (3,0) circle (1 pt);
            \filldraw (Cs) + (3,0) circle (1 pt);
            \filldraw (At) + (3,0) circle (1 pt);
            \filldraw (Bt) + (3,0) circle (1 pt);
            \filldraw (Ct) + (3,0) circle (1 pt);

            % Transport plan 3
            \draw[thick, dashed, gray] (As) + (6,0) -- (6, 1.8);
            \draw[thick, dashed, gray] (Bs) + (6,0) -- (6, 1.2);
            \draw[thick, dashed, gray] (Cs) + (6,0) -- (6, 0.6);

            \filldraw (As) + (6,0) circle (1 pt);
            \filldraw (Bs) + (6,0) circle (1 pt);
            \filldraw (Cs) + (6,0) circle (1 pt);
            \filldraw (At) + (6,0) circle (1 pt);
            \filldraw (Bt) + (6,0) circle (1 pt);
            \filldraw (Ct) + (6,0) circle (1 pt);

            % Transport plan 4
            \draw[thick, dashed, gray] (As) + (0,-3) -- (0, -1.8);
            \draw[thick, dashed, gray] (Bs) + (0,-3) -- (0, -2.4);
            \draw[thick, dashed, gray] (Cs) + (0,-3) -- (0, -1.2); 

            \filldraw (As) + (0,-3) circle (1 pt);
            \filldraw (Bs) + (0,-3) circle (1 pt);
            \filldraw (Cs) + (0,-3) circle (1 pt);
            \filldraw (At) + (0,-3) circle (1 pt);
            \filldraw (Bt) + (0,-3) circle (1 pt);
            \filldraw (Ct) + (0,-3) circle (1 pt);

            % Transport plan 5
            \draw[thick, dashed, gray] (As) + (3,-3) -- (3, -1.8);
            \draw[thick, dashed, gray] (Bs) + (3,-3) -- (3, -1.2);
            \draw[thick, dashed, gray] (Cs) + (3,-3) -- (3, -2.4); 

            \filldraw (As) + (3,-3) circle (1 pt);
            \filldraw (Bs) + (3,-3) circle (1 pt);
            \filldraw (Cs) + (3,-3) circle (1 pt);
            \filldraw (At) + (3,-3) circle (1 pt);
            \filldraw (Bt) + (3,-3) circle (1 pt);
            \filldraw (Ct) + (3,-3) circle (1 pt);

            % Transport plan 6
            \draw[thick, dashed, gray] (As) + (6,-3) -- (6, -1.2);
            \draw[thick, dashed, gray] (Bs) + (6,-3) -- (6, -2.4);
            \draw[thick, dashed, gray] (Cs) + (6,-3) -- (6, -1.8); 

            \filldraw (As) + (6,-3) circle (1 pt);
            \filldraw (Bs) + (6,-3) circle (1 pt);
            \filldraw (Cs) + (6,-3) circle (1 pt);
            \filldraw (At) + (6,-3) circle (1 pt);
            \filldraw (Bt) + (6,-3) circle (1 pt);
            \filldraw (Ct) + (6,-3) circle (1 pt);

            % Costs
            \node[anchor=south] at (1, 2) {$\text{Cost} \approx \hspace{1 mm} 8.485$};
            \node[anchor=south] at (4, 2) {$\text{Cost} \approx \hspace{1 mm} 8.625$};
            \node[anchor=south] at (6.5, 2) {$\text{Cost} \approx \hspace{1 mm} 9.152$};

            \node[anchor=south] at (1, -1) {$\text{Cost} \approx \hspace{1 mm} 8.719$};
            \node[anchor=south] at (4, -1) {$\text{Cost} \approx \hspace{1 mm} 9.003$};
            \node[anchor=south] at (6.5, -1) {$\text{Cost} \approx \hspace{1 mm} 9.003$};
            \end{tikzpicture}
    \end{center}
    \caption{When the sets $\mathcal{D}^s$ and $\mathcal{D}^t$ has $3$ points there are $3!=6$ possibles transport map. In the image we show all possibilities and their costs using $||\cdot||_2$ as cost function. Clearly the optimal transport map is the one with out crossing and is exactly the rotation.}\label{fig:example_n_3}
\end{figure}
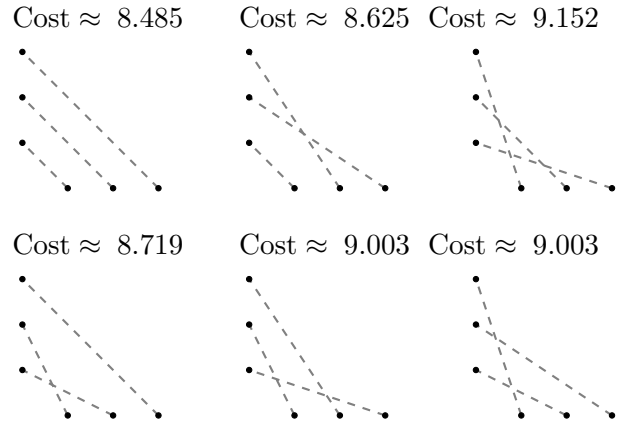
\subsection{Domain Adaptation framework}
We now consider a domain adaptation setting with two datasets in $\mathbb{R}^2$: $\mathcal{D}^s = \{(x_i^s, y_i^s)\}_{i=1}^{n_s}$ and $\mathcal{D}^t = \{(x_i^t, y_i^t)\}_{i=1}^{n_t}$, where $n_s \gg n_t$, as in transfer learning. We assume both lie closed to straight lines related by an unknown transformation $f \in \textbf{MOV}$. The source $\mathcal{D}^s$ is used for training, while $\mathcal{D}^t$ represents a target domain where data is scarce due to cost or physical limitations, reflecting common challenges of data drift as discussed in \cite{Sahiner}. Our goal is to adapt a linear regression model trained on $\mathcal{D}^s$ to  $\mathcal{D}^t$.

Although each type of $f \in \textbf{MOV}$ (rotation, translation, or homothety) requires a specific treatment presented in Algorithm \ref{alg:common_step}: apply $K$-means (\cite{K-means}) with $K = n_t$ on $\mathcal{D}^s$ to get a centroid set $\mathcal{D}^k$, then use optimal transport to align $\mathcal{D}^k$ with $\mathcal{D}^t$. This procedure is illustrated in Figure~\ref{fig:R2} for a rotation of angle $\theta = \frac{\pi}{4}$, and is motivated by Proposition~\ref{T_equal_f}, since $|\mathcal{D}^k| = |\mathcal{D}^t|$. While $\mathcal{D}^k$ and $\mathcal{D}^t$ may not be strictly related by $f$, the optimal transport map $T$ can approximate $f$ under suitable conditions (see Section \ref{conclusion}).

\begin{figure}[ht!]
    \centering
    \includegraphics[width=1\linewidth]{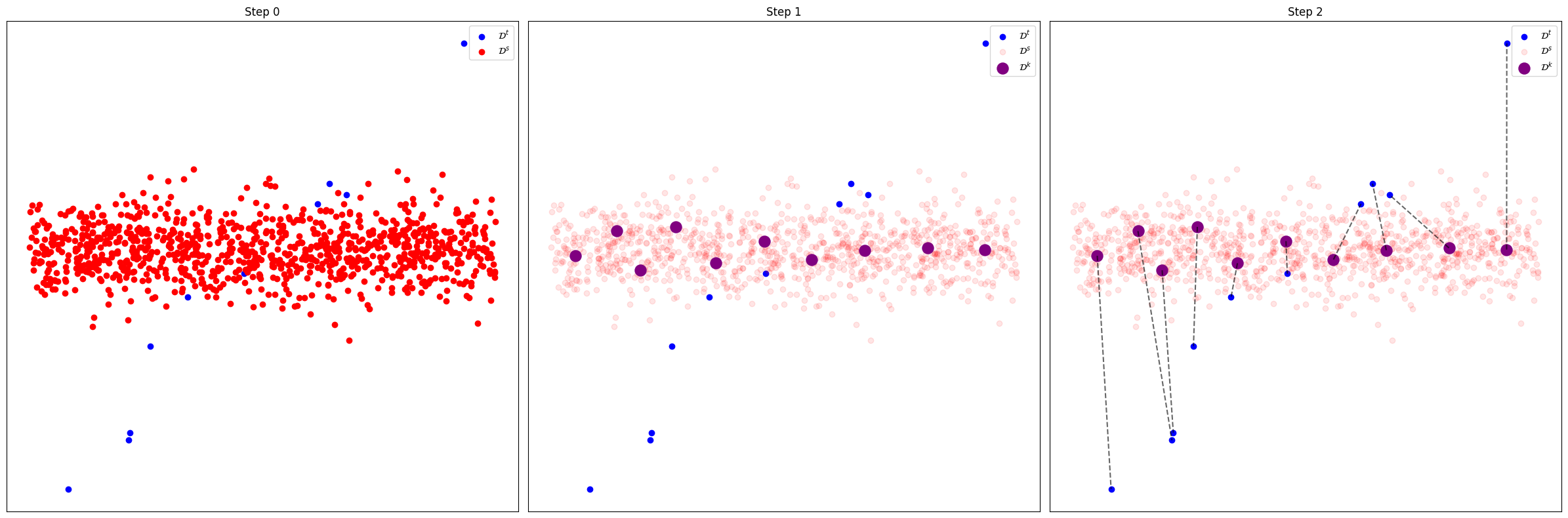}
    \caption{Red points are the source domain, blue ones the target domain and the purple dots are the $K$-means centroids. Step 0 is the raw data, in Step 1 we apply $K$-means with $K = n_t$ and in Step 3 we find the optimal transport map.}
    \label{fig:R2}
\end{figure}

\begin{algorithm}[htbp]
\caption{Common step using $K$-Means and Optimal Transport}
\label{alg:common_step}
\begin{algorithmic}
\REQUIRE $\mathcal{D}^s, \mathcal{D}^t$: Source and target domain.
\ENSURE Alignment of the domains $\mathcal{D}^k$ and $\mathcal{D}^t$.
\STATE \textbf{Step 1:} $\mathcal{D}^k \gets \text{KMeans}(\mathcal{D}_s, \text{n\_clusters} = n_t)$.
\STATE \textbf{Step 2:} $T \gets$ Optimal transport map obtained with uniform measures $\mu$, $\nu$ and $|| \cdot ||_p$.
\STATE \textbf{Step 3:} Use the optimal transport map $T$ to allying the sets $\mathcal{D}^t$ and $\mathcal{D}^k$.

\RETURN 
\end{algorithmic}
\end{algorithm}

In the following three subsections we propose the different approach for each possible map in $\textbf{MOV}$.
\subsubsection{Rotation}
After aligning the sets $\mathcal{D}^k$ and $\mathcal{D}^t$ we use the algorithm proposed in \cite{arun1987} to estimate the rotation angle $\theta$ show in Algorithm \ref{alg:rotation_angle}.

\begin{algorithm}[H]
\caption{Estimate rotation angle $\hat{\theta}$ (\cite{arun1987})}
\label{alg:rotation_angle}
\begin{algorithmic}
\REQUIRE Two point sets $\mathcal{D}^s = \{x_j^s\}_{j=1}^n$, $\mathcal{D}^t = \{x_j^t\}_{j=1}^n$ in $\mathbb{R}^2$
\ENSURE Estimated angle $\hat{\theta}$

\STATE Compute centroids: $\bar{x}^s = \frac{1}{n} \sum x_j^s$, $\bar{x}^t = \frac{1}{n} \sum x_j^t$
\STATE Center points: $\tilde{x}_j^s = x_j^s - \bar{x}^s$, $\tilde{x}_j^t = x_j^t - \bar{x}^t$
\STATE Compute $H = \sum \tilde{x}_j^s (\tilde{x}_j^t)^\top$
\STATE SVD: $H = U \Sigma V^\top$, set $\hat{R}_{\theta} = VU^\top$
\STATE If $\det(\hat{R}_{\theta}) = -1$, flip last column of $V$ and recompute $\hat{R}_{\theta} = VU^\top$
\STATE Let $z$ be the first column of $\hat{R}_{\theta}$.
\STATE $\hat{\theta} \gets \arg(z)$
\RETURN $\hat{\theta}$
\end{algorithmic}
\end{algorithm}

\subsubsection{Translation} \label{translation_math}
After alignment $\mathcal{D}^t$ and $\mathcal{D}^k$, each point satisfies $(x_i^t, y_i^t) = (x_i^k + q_{1i}, y_i^k + q_{2i})$. Thus, the translation vector per point is  
\[
(q_{1i}, q_{2i}) = (x_i^k, y_i^k) - (x_i^t, y_i^t).
\]
In Proposition \ref{T_equal_f}, we assume constant translation: $q_{1i} = q_1$, $q_{2i} = q_2$ for all $i$. However, since $\mathcal{D}^k$ and $\mathcal{D}^t$ are not perfectly aligned, we estimate $(q_1, q_2)$ as the average of $\{(q_{1i}, q_{2i})\}_{i=1}^{n_t}$.

\subsubsection{Homothety} \label{homothety_math}
The domains $\mathcal{D}^k$ and $\mathcal{D}^t$ satisfy $(x_i^t, y_i^t) = (\lambda_{{x_{i}}} x_i^k, \lambda_{{y_{i}}}  y_i^k)$. We estimate the scaling factors as 
$$
(\lambda_{{x_{i}}}, \lambda_{{y_{i}}}) = (x_i^t, y_i^t) \oslash (x_i^k, y_i^k),
$$
where $\oslash$ denotes element-wise division. Then we compute $\lambda_i=\frac{\lambda_{x_i} + \lambda_{y_i}}{2}$ for each $i$ and finally we estimate the homothety factor $\lambda$ as the mean of all the $\{\lambda_i\}_{i=1}^{n_t}$. This is a usual approach when estimating a homotecy constants.

The propose method to estimate the map $f$ relating the domains $\mathcal{D}^s$ and $\mathcal{D}^t$ is presented in Algorithm \ref{alg:estimacion_del_angulo_SVD}.

\begin{algorithm}[ht!]
\caption{Estimation of the map $f \in \textbf{MOV}$}
\label{alg:estimacion_del_angulo_SVD}
\begin{algorithmic}
\REQUIRE $\mathcal{D}^s, \mathcal{D}^t$: Source and target domain, type of transformation in $\textbf{MOV}$.
\ENSURE $\hat{f}$ estimation of $f$.
\STATE \textbf{Step 1:} Apply algorithm \ref{alg:common_step} to aligning the domains $\mathcal{D}^k$ and $\mathcal{D}^t$.
\STATE \textbf{Step 2:} IF (type == rotation):
\STATE \hspace{1.6 cm} $\hat{\theta} \gets$ Estimate rotation angle between $\mathcal{D}^s$ and $\mathcal{D}^k$\\  \hspace{1.6 cm} using \citet{arun1987}.
\STATE \hspace{1.1 cm} IF (type == translation)
\STATE \hspace{1.6 cm} $(\hat{q_1}, \hat{q_2}) \gets \text{mean} \Big( \Big\{ (q_{1i}, q_{2i})  \Big\}_{i=1}^{n_t} \Big)$, where \STATE 
\hspace{1.6 cm} $(q_{1i}, q_{2i})$ are calculated as in \ref{translation_math}.

\STATE \hspace{1.1 cm} IF (type == homothety)
\STATE \hspace{1.6 cm} $\hat{\lambda} \gets \text{mean} \Big( \Big\{ \lambda_i \Big\}_{i=1}^{n_t} \Big)$, where $\lambda_i$ is calculated as
\STATE \hspace{1.6 cm} in \ref{homothety_math}.
\end{algorithmic}
\end{algorithm}

After estimating the transformation $f \in \textbf{MOV}$, we adapt a linear regression trained on $\mathcal{D}^s$ to the target domain $\mathcal{D}^t$ using Lemma~\ref{relation_between_parameters}. We repeat this process $N$ times and take the median of the estimated parameters $\hat{a}$ and $\hat{b}$. To promote centroid variability in $K$-means, we apply a bootstrap step before estimating $f$ via Algorithm~\ref{alg:estimacion_del_angulo_SVD}. The full procedure, Adapted Linear Regression in $\mathbb{R}^2$ (ALR), is summarized in Algorithm~\ref{alg:rotated_regression}. Taking medians of means is a common approach, as discussed in Section 2.3 of \cite{lerasle2019lecture}.

\begin{algorithm}[ht!]
\caption{Adapted Linear Regression in $\mathbb{R}^2$ (ALR)}
\label{alg:rotated_regression}
\begin{algorithmic}
\REQUIRE $\mathcal{D}^s, \mathcal{D}^t$: source and target domains; $N$: number of repetitions; $p$: bootstrap proportion, $\texttt{type}$: type of function $f \in \textbf{MOV}$.
\ENSURE $a_r, b_r$: adapted regression coefficients.
\STATE Fit linear regression on $\mathcal{D}^s$ to obtain $a_s, b_s$
\FOR{$i = 1$ to $N$}
    \STATE $\mathcal{D}^i \gets$ Sample bootstrap of subset $\mathcal{D}^s$ (proportion $p$).
    \STATE $\hat{f}$ $\gets$ Apply Algorithm~\ref{alg:estimacion_del_angulo_SVD} to $\mathcal{D}^t$ and $\mathcal{D}^i$.
    \STATE $a_i, b_i \gets$ Apply lemma \ref{relation_between_parameters} with $a_s, b_s$, and $\hat{f}$ parameters depending on \texttt{type}.
\ENDFOR
\RETURN $a_r = \text{median}(a_1, \ldots, a_N)$, $b_r = \text{median}(b_1, \ldots, b_N)$
\end{algorithmic}
\end{algorithm}

\section{Simulations}
\label{sec:sim}
In this section we present a series of simulations to compare the proposed procedure (which we will called \textbf{ALR} for \textit{Adapted Linear Regression}) against three established benchmarks:
\begin{itemize}
    \item \textbf{Target-Only (TO):} A standard baseline where the model is trained exclusively on the target domain data, representing the performance without any knowledge transfer \cite{vapnik1999overview}.
    \item \textbf{Linear Transfer Learning (LFT):} A common approach where features learned from a source task are kept fixed, and only a linear transformation or output layer is optimized for the target data \cite{yosinski2014transferable}.
    \item \textbf{Trans-Lasso (TL):} A robust transfer learning method designed for high-dimensional linear regression that leverages informative source samples to improve the estimation of the target parameter vector \cite{li2022transfer}.
\end{itemize}
These tests illustrate its practical advantages and highlight scenarios of significant predictive improvement. This section is split in 3 parts: simulations for the rotation, simulations for the translation and simulations for the homothety. We use Python for the simulations which have a great library to work with optimal transport problems due to \cite{JMLR:v22:20-451}. In all the cases we use the $2-$norm as cost function and uses the same noise level because we are interested in the shift problem. 

\bigskip
We begin by exploring two rotation scenarios: in \ref{var_DS}, we vary the source domain size $n_s$ while keeping the angle, noise level, and target size $n_t$ fixed; in \ref{var_theta_sigma}, we vary the angle and noise while keeping $n_s$ and $n_t$ fixed. Then, in \ref{simulation_translation}, we analyze how changes in the and norm of the translation vector affect performance. Finally in \ref{simulation_homothety}, we study the impact of the homothety constant. Overall, our method generally outperforms fitting linear regression directly on the target domain, prompting a second experiment for deeper analysis.

Simulations were conducted by sampling from a line with Gaussian noise: given $a \in \mathbb{R}$ and noise variance $\sigma^2 > 0$, we generate data as $y = ax + \epsilon$, with $\epsilon \sim \mathcal{N}(0, \sigma^2)$. Exploiting rotational symmetry, we align the source domain with the $x$-axis by setting $a = 0$, i.e., $y = \epsilon$. The target domain varies depending on the transformation $f \in \textbf{MOV}$.

\subsection{Rotation simulations}

In the rotation situation, we generate the target domain $\mathcal{D}^t$ fixing an angle $\theta \in [0, \pi)$ and generate the points following $y = \tan(\theta) x + \epsilon$, where $\epsilon \sim \mathcal{N}(0, \sigma^2)$ with the noise variance $\sigma^2$ the same used for the source domain.

\subsubsection{Varying the cardinality of $\mathcal{D}^s$} \label{var_DS}
To evaluate the performance of the proposed method relative to the ratio between target and source domain sizes ($n_t$ and $n_s$), we conducted a simulation using a fixed rotation angle $\theta = \frac{\pi}{4}$, target size $n_t = 10$, and noise level $\sigma^2 = 1$. We varied the source domain size $n_s$ across several orders of magnitude, ranging from $10^1$ to $10^6$. For each value of $n_s$, we executed $1,000$ trials and recorded the median Mean Square Error (MSE), as shown in Figure~\ref{fig:exp_1}. 

The results indicate that while our approach offers no significant advantage when $n_s \approx n_t$, it consistently outperforms the baseline as the ratio $n_s/n_t$ increases, eventually exhibiting asymptotic behavior. Despite a higher frequency of outliers, the lower median MSE suggests that the proposed method is more robust and provides more reliable results in the majority of scenarios.

\begin{figure}[ht!]
    \centering
    \includegraphics[width=\linewidth]{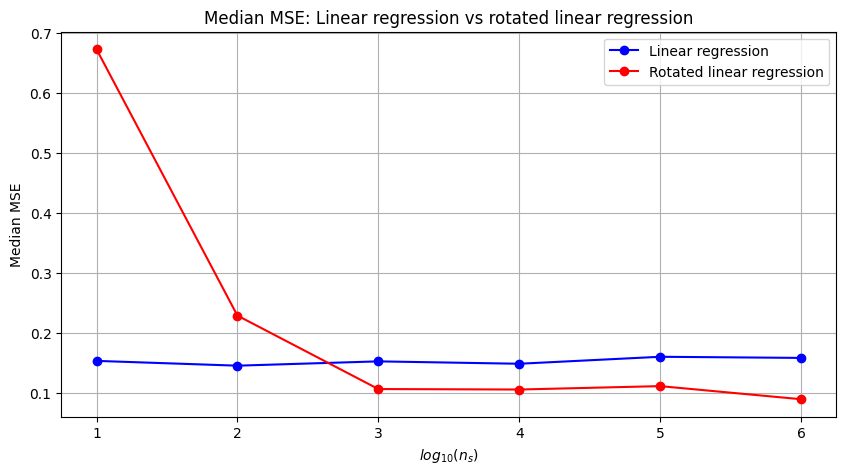}
        \caption{The target domain size $n_t$ is fixed at 10, while the source domain size $n_s$ varies in $\{10, 10^2, 10^3, 10^4, 10^5, 10^6\}$. Blue dots show the median MSE for regression on $\mathcal{D}^t$, and red points show the median MSE using our method.}
        \label{fig:exp_1}
\end{figure}

\subsubsection{Varying the angle $\theta$ and noise variance $\sigma^2$} \label{var_theta_sigma}
To investigate the operational boundaries of the proposed method, we fixed the domain sizes at $n_s=1000$ and $n_t = 50$ and evaluated performance across the $\theta$-$\sigma^2$ plane. We tested a comprehensive range of rotation angles $\theta \in \{\frac{\pi}{6}, \frac{\pi}{4}, \frac{\pi}{3}, \frac{9 \pi}{20}, \frac{11 \pi}{20}, \frac{2\pi}{3}, \frac{3\pi}{4}, \frac{5\pi}{6}\}$ and noise levels $\sigma^2 \in \{0.1, 0.2, 0.5, 1, 2, 5\}$. 

\bigskip
To evaluate performance, 100 independent trials were performed for each $\theta$-$\sigma^2$ configuration. We benchmarked each method against the ground truth using the median Mean Squared Error (MSE). Table~\ref{tab:main_results} displays some relevant results, showing that the proposed method outperforms the another tested methods at low noise levels, effectively leveraging the larger source domain to provide stability. The complete table is in the Appendix.

\begin{table}[ht!]
\centering
\small
\setlength{\tabcolsep}{3pt}
\renewcommand{\arraystretch}{1.1}

\resizebox{\columnwidth}{!}{
\begin{tabular}{c|cccc}
\hline
$\sigma^2$ & TO & LFT & TL & ALR \\ 
\hline

\multicolumn{5}{c}{$\theta=\frac{\pi}{6}$} \\ 
\hline
0.1 & $6.00{\times}10^{-4} \pm 6.00{\times}10^{-4}$ 
    & $5.00{\times}10^{-4} \pm 6.00{\times}10^{-4}$ 
    & $1.30{\times}10^{-3} \pm 1.20{\times}10^{-3}$ 
    & \cellcolor{gray!15}$\mathbf{4.00{\times}10^{-4} \pm 7.00{\times}10^{-4}}$ \\
0.5 & $2.08{\times}10^{-2} \pm 2.60{\times}10^{-2}$ 
    & \cellcolor{gray!15}$\mathbf{1.81{\times}10^{-2} \pm 2.26{\times}10^{-2}}$ 
    & $2.13{\times}10^{-2} \pm 2.54{\times}10^{-2}$ 
    & $2.14{\times}10^{-2} \pm 4.17{\times}10^{-2}$ \\
2.0 & $5.41{\times}10^{-1} \pm 3.99{\times}10^{-1}$ 
    & \cellcolor{gray!15}$\mathbf{5.18{\times}10^{-1} \pm 3.86{\times}10^{-1}}$ 
    & $5.69{\times}10^{-1} \pm 4.11{\times}10^{-1}$ 
    & $3.52{\times}10^{0} \pm 1.25{\times}10^{0}$ \\

\hline
\multicolumn{5}{c}{$\theta=\frac{1}{3}\pi$} \\ 
\hline
0.1 & $4.50{\times}10^{-3} \pm 5.40{\times}10^{-3}$ 
    & $4.90{\times}10^{-3} \pm 5.70{\times}10^{-3}$ 
    & $1.14{\times}10^{-2} \pm 9.40{\times}10^{-3}$ 
    & \cellcolor{gray!15}$\mathbf{3.90{\times}10^{-3} \pm 5.20{\times}10^{-3}}$ \\
0.5 & $2.24{\times}10^{-1} \pm 2.30{\times}10^{-1}$ 
    & \cellcolor{gray!15}$\mathbf{2.21{\times}10^{-1} \pm 2.29{\times}10^{-1}}$ 
    & $2.77{\times}10^{-1} \pm 2.64{\times}10^{-1}$ 
    & $4.39{\times}10^{-1} \pm 5.33{\times}10^{-1}$ \\
2.0 & $1.13{\times}10^{1} \pm 5.11{\times}10^{0}$ 
    & \cellcolor{gray!15}$\mathbf{1.12{\times}10^{1} \pm 5.01{\times}10^{0}}$ 
    & $1.16{\times}10^{1} \pm 5.16{\times}10^{0}$ 
    & $5.69{\times}10^{1} \pm 7.58{\times}10^{0}$ \\

\hline
\multicolumn{5}{c}{$\theta=\frac{9}{20}\pi$} \\ 
\hline
0.1 & \cellcolor{gray!15}$\mathbf{5.43{\times}10^{-1} \pm 8.22{\times}10^{-1}}$ 
    & $6.69{\times}10^{1} \pm 2.94{\times}10^{1}$ 
    & $1.83{\times}10^{0} \pm 1.48{\times}10^{0}$ 
    & $5.90{\times}10^{-1} \pm 1.00{\times}10^{0}$ \\
0.5 & \cellcolor{gray!15}$\mathbf{8.38{\times}10^{1} \pm 3.63{\times}10^{1}}$ 
    & $1.63{\times}10^{2} \pm 4.84{\times}10^{1}$ 
    & $9.43{\times}10^{1} \pm 3.84{\times}10^{1}$ 
    & $2.29{\times}10^{2} \pm 5.33{\times}10^{1}$ \\
2.0 & \cellcolor{gray!15}$\mathbf{9.72{\times}10^{2} \pm 1.36{\times}10^{2}}$ 
    & $9.72{\times}10^{2} \pm 1.35{\times}10^{2}$ 
    & $9.80{\times}10^{2} \pm 1.36{\times}10^{2}$ 
    & $1.29{\times}10^{3} \pm 5.77{\times}10^{1}$ \\

\hline
\end{tabular}
}
\caption{Mean $\pm$ standard deviation of MSE across representative regimes. ALR performs best in low-noise settings but degrades significantly as noise increases, both in accuracy and stability.}
\label{tab:main_results}
\end{table}

\subsection{Translation simulations} \label{simulation_translation}
In the translation setting, we fix $a=0$, $\sigma^2=1$, $n_s=1000$, and $n_t=10$, and test translation norms $1,2,4,$ and $16$. For each norm, we run 100 trials by generating the domains, applying a random translation of the given norm to $\mathcal{D}^t$, and applying our method. This produces 100 MSE values for both the baseline and our approach. The results are summarized in Figure~\ref{fig:translation_simulation}.

\begin{figure}[ht!]
    \centering
    \includegraphics[width=.85\linewidth]{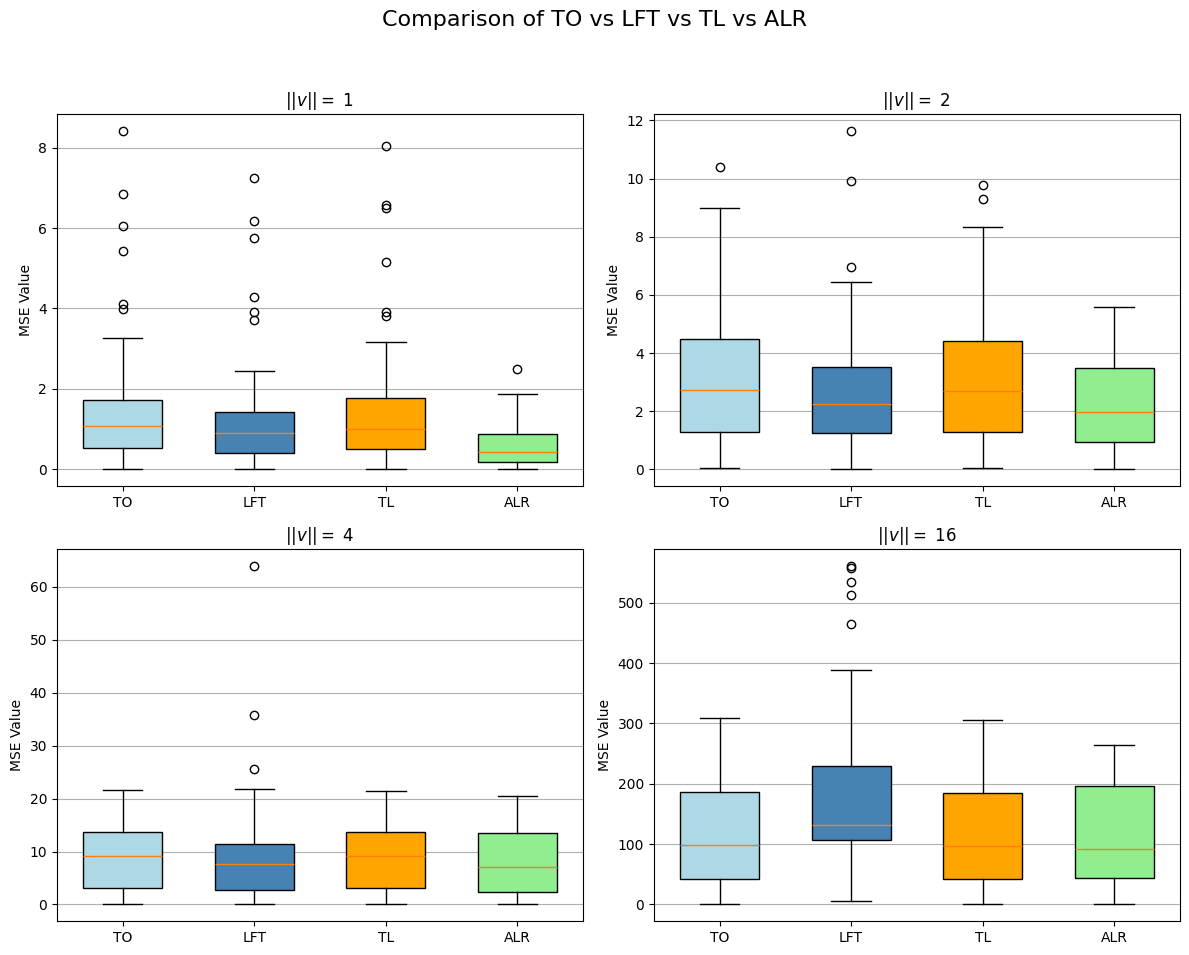}
    \caption{Boxplots of MSE in the translation setting for varying norms of the translation vector $\|v\|$. When the magnitude of the translation is small ($\|v\| < 2$), the proposed method achieves the lowest median error among all methods. However, its relative performance degrades as the norm of the translation increases.}
    \label{fig:translation_simulation}
\end{figure}

\newpage

\subsection{Homothety simulation} \label{simulation_homothety}
In the homothety situation, to generate the target domain $\mathcal{D}^t$ we first draw a set $\mathcal{D}_0^t$ using the same values for $a$ and $\sigma^2$ as for the source domain $\mathcal{D}^s$ and fix $n_t=10$. Then we apply the homothety constant $\lambda$: $\mathcal{D}^t = \lambda \mathcal{D}_0^t = \{\lambda(x,y) \hspace{2 mm} | \hspace{2 mm} (x,y) \in \mathcal{D}_0^t \}$.

\bigskip
To test the method with homothecies, we select for values for the homothety constant $\lambda$: $-2, 0.5, 2$ and $4$. For each fixed $\lambda$, we repeat $100$ times the simulation: first we draw the source and the target domain and then multiply $\mathcal{D}^t$ by $\lambda$. Table \ref{table:homothety} present the mean and variance for the four methods.

\begin{table}[ht!]
\centering
\small
\setlength{\tabcolsep}{3pt}
\renewcommand{\arraystretch}{1.1}

\resizebox{\columnwidth}{!}{
\begin{tabular}{c|cccc}
\hline
$\lambda$ & TO & LFT & TL & ALR \\ 
\hline

\multicolumn{5}{c}{Homothety experiment} \\ 
\hline

-2 & $1.69{\times}10^{-2} \pm 2.33{\times}10^{-4}$ 
   & $1.50{\times}10^{-2} \pm 2.08{\times}10^{-4}$ 
   & $1.36{\times}10^{-2} \pm 1.74{\times}10^{-4}$ 
   & \cellcolor{gray!15}$\mathbf{1.13{\times}10^{-4} \pm 1.75{\times}10^{-8}}$ \\
-1 & $4.20{\times}10^{-3} \pm 4.27{\times}10^{-5}$ 
   & $3.72{\times}10^{-3} \pm 3.42{\times}10^{-5}$ 
   & $2.10{\times}10^{-3} \pm 1.16{\times}10^{-5}$ 
   & \cellcolor{gray!15}$\mathbf{3.96{\times}10^{-5} \pm 1.46{\times}10^{-9}}$ \\
0.5 & $1.26{\times}10^{-3} \pm 2.17{\times}10^{-6}$ 
    & $1.11{\times}10^{-3} \pm 1.70{\times}10^{-6}$ 
    & $2.43{\times}10^{-4} \pm 9.61{\times}10^{-8}$ 
    & \cellcolor{gray!15}$\mathbf{4.43{\times}10^{-5} \pm 2.35{\times}10^{-9}}$ \\
4 & $8.87{\times}10^{-2} \pm 2.00{\times}10^{-2}$ 
  & $7.82{\times}10^{-2} \pm 1.56{\times}10^{-2}$ 
  & $8.45{\times}10^{-2} \pm 1.85{\times}10^{-2}$ 
  & \cellcolor{gray!15}$\mathbf{4.75{\times}10^{-4} \pm 3.26{\times}10^{-7}}$ \\

\hline
\end{tabular}
}
\caption{Mean $\pm$ variance across different $\lambda$ values. Bold values indicate the minimum (best) result for each row. For the tested $\lambda$ values, the proposed method achieves better results.}
\label{table:homothety}
\end{table}

To gain further insight into these results, we perform an additional simulation. We use the same parameters ($n_s=1000$, $n_t=10$, $\lambda=2$), but first sample $1000$ points and then randomly select $10$ of them to define the target domain $\mathcal{D}^t$. The experiment is then carried out while retaining the full set of $1000$ points to visualize the underlying structure. From the upper panel of Figure~\ref{fig:complete_example}, the yellow line appears to better fit the target data than the green one. However, when the full sample is shown in the lower panel, we observe that the true underlying structure is a horizontal line (consistent with $a=0$), indicating that the green line better captures the global structure.
\begin{figure}[ht!]
    \centering
    \includegraphics[width=1\linewidth]{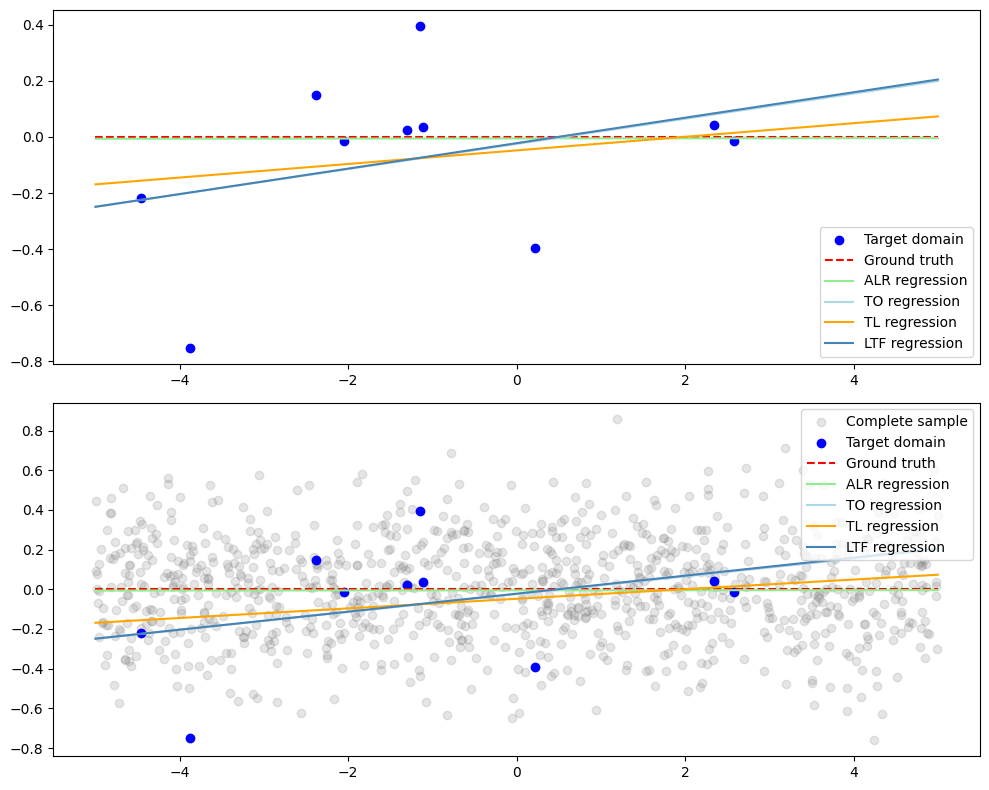}
    \caption{\textbf{Up:} Simulation with $n_s=1000$, $n_t=10$, $\sigma^2=1$, and $\lambda=0.5$. While the TO, LT, and LTF methods approximate the target domain well, our ADL method better captures the underlying structure. \textbf{Down:} The grey points are the complete underlying structure.}
    \label{fig:complete_example}
\end{figure}

\subsection{Limitations}
In the rotation case, we observe experimentally that the proposed method is not effective when the rotation angle $\theta$ is close to $0$ or $\pi$, or when the noise variance $\sigma^2$ is big. For translations, performance deteriorates when the translation norm is large: after applying $K$-means, the resulting sets $\mathcal{D}^k$ and $\mathcal{D}^t$ are not exactly related by a translation in $\textbf{MOV}$, so the segments connecting corresponding points are no longer parallel, and this discrepancy grows with the translation magnitude. In the homothety setting, we do not observe cases where our method performs worse; however, we consider this scenario less relevant in practice, as rotations and translations arise more frequently in real applications.

\bigskip
Another limitation is that the composition of optimal transport maps is generally not optimal. Therefore, if two domains are related by a composition of maps $f_1, \ldots, f_n \in \textbf{MOV}$, repeatedly applying our method does not necessarily yield meaningful results. For example, consider the set $\{(1,0),(2,0)\}$ with $f_1$ a rotation by $\pi$ and $f_2$ a translation by $q=(3,0)$. Then
$$
\{(1,0),(2,0)\} \xrightarrow{f_1} \{(-1,0),(-2,0)\} \xrightarrow{f_2} \{(1,0),(2,0)\},
$$
so that $f_2 \circ f_1 = \mathrm{Id}$ is the optimal transport map, from which neither the rotation nor the translation can be recovered.

\section{Conclusion and perspectives} \label{conclusion}
We propose a geometric approach to domain adaptation for linear regression in $\mathbb{R}^2$, leveraging optimal transport under rotations, translations, and homotheties. By combining OT with $K$-means and simple estimators, we infer the geometric relationship between source and target domains and adapt models using very few target labels. Simulations show consistent median improvements over direct regression (TO), Linear Transfer Learning (LTF) and Trans-Lasso (TL) on scarce target data, particularly with large source samples and high noise, highlighting the potential of OT-based geometric alignment for regression under domain shift. We focused on $\mathbb{R}^2$ to better understand the underlying geometry of the problem. As future work, we plan to explore the extension of these results to $\mathbb{R}^n$, where additional geometric and computational challenges may arise. Possible next steps are discussed below:

\begin{itemize}
    \item \textbf{Understanding outliers and robustness:}  The proposed method shows strong median performance but exhibits significant outliers, likely due to domain transformations that deviate substantially from any map $f \in \textbf{MOV}$. Analyzing these cases is key to developing a more robust algorithm.
    \item \textbf{Generalization:} while we focus on a single global linear transformation, we note that extensions based on piecewise or polynomial models may allow the framework to address more complex mappings; exploring these ideas is left for future work.

    \item \textbf{Generalization to higher dimensions:} although the proposed method shows encouraging results for $\mathbb{R}^2$, real world data generally have dimensions greater than $2$. One possible approach is to use shape analysis techniques (\cite{srivastava2016functional}) to align the target domain with the domain obtained after applying K-means, in order to estimate the transformation between the domains. Then, it would be necessary to study how lemma \ref{relation_between_parameters} change in the $\mathbb{R}^n$ case.

    \item \textbf{Add a hypothesis test}: despite encouraging results, in the propose approach it's necessary to know beforehand which type of transformation $f \in \textbf{MOV}$ related the domains. An interesting next-step is to add a hypothesis test so the algorithm can also estimate which transformation is needed, following ideas considered in shape analysis (\cite{mansour2023domain}, \cite{srivastava2016functional}).

\end{itemize}

These extensions, combined with the demonstrated effectiveness of our method in \( \mathbb{R}^2 \), provide a foundation for future work aimed at addressing more complex domain adaptation challenges in higher dimensions and diverse scenarios.

\bibliographystyle{model5-names}
\bibliography{refs}

\section*{Appendix}
\begin{table}[ht!]
\centering
\small
\setlength{\tabcolsep}{3pt}
\renewcommand{\arraystretch}{1.1}

\resizebox{\columnwidth}{!}{
\begin{tabular}{c|cccc}
\hline
$\sigma^2$ & TO & LFT & TL & ALR \\
\hline

\multicolumn{5}{c}{$\theta=\frac{\pi}{6}$} \\
\hline
0.1 & $6.00{\times}10^{-4}\pm6.00{\times}10^{-4}$ & $5.00{\times}10^{-4}\pm6.00{\times}10^{-4}$ & $1.30{\times}10^{-3}\pm1.20{\times}10^{-3}$ & \cellcolor{gray!15}$\mathbf{4.00{\times}10^{-4}\pm7.00{\times}10^{-4}}$ \\
0.2 & $2.60{\times}10^{-3}\pm2.90{\times}10^{-3}$ & $2.30{\times}10^{-3}\pm2.80{\times}10^{-3}$ & $3.50{\times}10^{-3}\pm3.90{\times}10^{-3}$ & \cellcolor{gray!15}$\mathbf{2.30{\times}10^{-3}\pm3.20{\times}10^{-3}}$ \\
0.5 & $2.08{\times}10^{-2}\pm2.60{\times}10^{-2}$ & \cellcolor{gray!15}$\mathbf{1.81{\times}10^{-2}\pm2.26{\times}10^{-2}}$ & $2.13{\times}10^{-2}\pm2.54{\times}10^{-2}$ & $2.14{\times}10^{-2}\pm4.17{\times}10^{-2}$ \\
1.0 & $6.59{\times}10^{-2}\pm6.67{\times}10^{-2}$ & \cellcolor{gray!15}$\mathbf{6.03{\times}10^{-2}\pm6.24{\times}10^{-2}}$ & $6.92{\times}10^{-2}\pm6.78{\times}10^{-2}$ & $2.32{\times}10^{-1}\pm4.03{\times}10^{-1}$ \\
2.0 & $5.41{\times}10^{-1}\pm3.99{\times}10^{-1}$ & \cellcolor{gray!15}$\mathbf{5.18{\times}10^{-1}\pm3.86{\times}10^{-1}}$ & $5.69{\times}10^{-1}\pm4.11{\times}10^{-1}$ & $3.52{\times}10^{0}\pm1.25{\times}10^{0}$ \\
5.0 & $8.43{\times}10^{0}\pm4.08{\times}10^{0}$ & \cellcolor{gray!15}$\mathbf{8.27{\times}10^{0}\pm4.05{\times}10^{0}}$ & $8.48{\times}10^{0}\pm4.03{\times}10^{0}$ & $1.08{\times}10^{1}\pm1.65{\times}10^{0}$ \\

\hline
\multicolumn{5}{c}{$\theta=\frac{\pi}{4}$} \\
\hline
0.1 & $9.00{\times}10^{-4}\pm9.00{\times}10^{-4}$ & $8.00{\times}10^{-4}\pm7.00{\times}10^{-4}$ & $2.70{\times}10^{-3}\pm2.00{\times}10^{-3}$ & \cellcolor{gray!15}$\mathbf{7.00{\times}10^{-4}\pm9.00{\times}10^{-4}}$ \\
0.2 & $4.70{\times}10^{-3}\pm4.70{\times}10^{-3}$ & \cellcolor{gray!15}$\mathbf{4.40{\times}10^{-3}\pm4.60{\times}10^{-3}}$ & $6.90{\times}10^{-3}\pm7.10{\times}10^{-3}$ & $4.40{\times}10^{-3}\pm6.00{\times}10^{-3}$ \\
0.5 & $3.82{\times}10^{-2}\pm4.62{\times}10^{-2}$ & \cellcolor{gray!15}$\mathbf{3.61{\times}10^{-2}\pm4.61{\times}10^{-2}}$ & $4.67{\times}10^{-2}\pm5.47{\times}10^{-2}$ & $7.79{\times}10^{-2}\pm1.22{\times}10^{-1}$ \\
1.0 & $2.46{\times}10^{-1}\pm2.40{\times}10^{-1}$ & \cellcolor{gray!15}$\mathbf{2.40{\times}10^{-1}\pm2.42{\times}10^{-1}}$ & $2.75{\times}10^{-1}\pm2.58{\times}10^{-1}$ & $1.00{\times}10^{0}\pm6.42{\times}10^{-1}$ \\
2.0 & $1.94{\times}10^{0}\pm1.07{\times}10^{0}$ & \cellcolor{gray!15}$\mathbf{1.91{\times}10^{0}\pm1.07{\times}10^{0}}$ & $2.02{\times}10^{0}\pm1.10{\times}10^{0}$ & $1.25{\times}10^{1}\pm3.60{\times}10^{0}$ \\
5.0 & $2.59{\times}10^{1}\pm8.04{\times}10^{0}$ & \cellcolor{gray!15}$\mathbf{2.57{\times}10^{1}\pm7.88{\times}10^{0}}$ & $2.61{\times}10^{1}\pm7.94{\times}10^{0}$ & $3.45{\times}10^{1}\pm3.89{\times}10^{0}$ \\

\hline
\multicolumn{5}{c}{$\theta=\frac{\pi}{3}$} \\
\hline
0.1 & $4.50{\times}10^{-3}\pm5.40{\times}10^{-3}$ & $4.90{\times}10^{-3}\pm5.70{\times}10^{-3}$ & $1.14{\times}10^{-2}\pm9.40{\times}10^{-3}$ & \cellcolor{gray!15}$\mathbf{3.90{\times}10^{-3}\pm5.20{\times}10^{-3}}$ \\
0.2 & $1.38{\times}10^{-2}\pm1.62{\times}10^{-2}$ & \cellcolor{gray!15}$\mathbf{1.32{\times}10^{-2}\pm1.63{\times}10^{-2}}$ & $2.40{\times}10^{-2}\pm2.65{\times}10^{-2}$ & $1.35{\times}10^{-2}\pm1.88{\times}10^{-2}$ \\
0.5 & $2.24{\times}10^{-1}\pm2.30{\times}10^{-1}$ & \cellcolor{gray!15}$\mathbf{2.21{\times}10^{-1}\pm2.29{\times}10^{-1}}$ & $2.77{\times}10^{-1}\pm2.64{\times}10^{-1}$ & $4.39{\times}10^{-1}\pm5.33{\times}10^{-1}$ \\
1.0 & $1.46{\times}10^{0}\pm1.20{\times}10^{0}$ & \cellcolor{gray!15}$\mathbf{1.45{\times}10^{0}\pm1.20{\times}10^{0}}$ & $1.61{\times}10^{0}\pm1.25{\times}10^{0}$ & $6.25{\times}10^{0}\pm3.43{\times}10^{0}$ \\
2.0 & $1.13{\times}10^{1}\pm5.11{\times}10^{0}$ & \cellcolor{gray!15}$\mathbf{1.12{\times}10^{1}\pm5.01{\times}10^{0}}$ & $1.16{\times}10^{1}\pm5.16{\times}10^{0}$ & $5.69{\times}10^{1}\pm7.58{\times}10^{0}$ \\
5.0 & $8.76{\times}10^{1}\pm1.57{\times}10^{1}$ & \cellcolor{gray!15}$\mathbf{8.74{\times}10^{1}\pm1.56{\times}10^{1}}$ & $8.80{\times}10^{1}\pm1.55{\times}10^{1}$ & $1.03{\times}10^{2}\pm4.89{\times}10^{0}$ \\

\hline
\multicolumn{5}{c}{$\theta=\frac{9}{20}\pi$} \\
\hline
0.1 & \cellcolor{gray!15}$\mathbf{5.43{\times}10^{-1}\pm8.22{\times}10^{-1}}$ & $6.69{\times}10^{1}\pm2.94{\times}10^{1}$ & $1.83{\times}10^{0}\pm1.48{\times}10^{0}$ & $5.90{\times}10^{-1}\pm1.00{\times}10^{0}$ \\
0.2 & \cellcolor{gray!15}$\mathbf{3.41{\times}10^{0}\pm3.70{\times}10^{0}}$ & $7.20{\times}10^{1}\pm2.85{\times}10^{1}$ & $5.92{\times}10^{0}\pm5.37{\times}10^{0}$ & $5.35{\times}10^{0}\pm8.39{\times}10^{0}$ \\
0.5 & \cellcolor{gray!15}$\mathbf{8.38{\times}10^{1}\pm3.63{\times}10^{1}}$ & $1.63{\times}10^{2}\pm4.84{\times}10^{1}$ & $9.43{\times}10^{1}\pm3.84{\times}10^{1}$ & $2.29{\times}10^{2}\pm5.33{\times}10^{1}$ \\
1.0 & \cellcolor{gray!15}$\mathbf{4.20{\times}10^{2}\pm1.10{\times}10^{2}}$ & $4.48{\times}10^{2}\pm1.04{\times}10^{2}$ & $4.34{\times}10^{2}\pm1.11{\times}10^{2}$ & $9.40{\times}10^{2}\pm9.60{\times}10^{1}$ \\
2.0 & \cellcolor{gray!15}$\mathbf{9.72{\times}10^{2}\pm1.36{\times}10^{2}}$ & $9.72{\times}10^{2}\pm1.35{\times}10^{2}$ & $9.80{\times}10^{2}\pm1.36{\times}10^{2}$ & $1.29{\times}10^{3}\pm5.77{\times}10^{1}$ \\
5.0 & $1.35{\times}10^{3}\pm7.62{\times}10^{1}$ & \cellcolor{gray!15}$\mathbf{1.35{\times}10^{3}\pm7.62{\times}10^{1}}$ & $1.35{\times}10^{3}\pm7.47{\times}10^{1}$ & $1.36{\times}10^{3}\pm1.67{\times}10^{1}$ \\

\hline
\multicolumn{5}{c}{$\theta=\frac{11}{20}\pi$} \\
\hline
0.1 & \cellcolor{gray!15}$\mathbf{3.81{\times}10^{-1}\pm5.00{\times}10^{-1}}$ & $5.63{\times}10^{1}\pm2.18{\times}10^{1}$ & $1.28{\times}10^{0}\pm1.19{\times}10^{0}$ & $5.63{\times}10^{-1}\pm7.45{\times}10^{-1}$ \\
0.2 & \cellcolor{gray!15}$\mathbf{5.14{\times}10^{0}\pm4.56{\times}10^{0}}$ & $8.08{\times}10^{1}\pm2.52{\times}10^{1}$ & $8.47{\times}10^{0}\pm5.90{\times}10^{0}$ & $5.99{\times}10^{0}\pm7.76{\times}10^{0}$ \\
0.5 & \cellcolor{gray!15}$\mathbf{7.73{\times}10^{1}\pm3.27{\times}10^{1}}$ & $1.57{\times}10^{2}\pm4.17{\times}10^{1}$ & $8.75{\times}10^{1}\pm3.41{\times}10^{1}$ & $2.26{\times}10^{2}\pm6.60{\times}10^{1}$ \\
1.0 & \cellcolor{gray!15}$\mathbf{4.36{\times}10^{2}\pm8.03{\times}10^{1}}$ & $4.60{\times}10^{2}\pm7.91{\times}10^{1}$ & $4.50{\times}10^{2}\pm8.12{\times}10^{1}$ & $9.16{\times}10^{2}\pm8.24{\times}10^{1}$ \\
2.0 & \cellcolor{gray!15}$\mathbf{9.86{\times}10^{2}\pm1.08{\times}10^{2}}$ & $9.86{\times}10^{2}\pm1.07{\times}10^{2}$ & $9.94{\times}10^{2}\pm1.08{\times}10^{2}$ & $1.29{\times}10^{3}\pm4.02{\times}10^{1}$ \\
5.0 & $1.33{\times}10^{3}\pm6.17{\times}10^{1}$ & \cellcolor{gray!15}$\mathbf{1.33{\times}10^{3}\pm6.15{\times}10^{1}}$ & $1.33{\times}10^{3}\pm6.04{\times}10^{1}$ & $1.36{\times}10^{3}\pm1.65{\times}10^{1}$ \\

\hline
\multicolumn{5}{c}{$\theta=\frac{2\pi}{3}$} \\
\hline
0.1 & $3.40{\times}10^{-3}\pm3.70{\times}10^{-3}$ & $3.40{\times}10^{-3}\pm3.90{\times}10^{-3}$ & $1.02{\times}10^{-2}\pm9.00{\times}10^{-3}$ & \cellcolor{gray!15}$\mathbf{2.80{\times}10^{-3}\pm3.20{\times}10^{-3}}$ \\
0.2 & $1.64{\times}10^{-2}\pm1.79{\times}10^{-2}$ & $1.54{\times}10^{-2}\pm1.71{\times}10^{-2}$ & $2.43{\times}10^{-2}\pm2.09{\times}10^{-2}$ & \cellcolor{gray!15}$\mathbf{1.49{\times}10^{-2}\pm1.63{\times}10^{-2}}$ \\
0.5 & $1.79{\times}10^{-1}\pm1.79{\times}10^{-1}$ & \cellcolor{gray!15}$\mathbf{1.77{\times}10^{-1}\pm1.80{\times}10^{-1}}$ & $2.23{\times}10^{-1}\pm2.14{\times}10^{-1}$ & $3.69{\times}10^{-1}\pm3.93{\times}10^{-1}$ \\
1.0 & $1.54{\times}10^{0}\pm1.03{\times}10^{0}$ & \cellcolor{gray!15}$\mathbf{1.53{\times}10^{0}\pm1.03{\times}10^{0}}$ & $1.69{\times}10^{0}\pm1.08{\times}10^{0}$ & $7.25{\times}10^{0}\pm3.27{\times}10^{0}$ \\
2.0 & $1.32{\times}10^{1}\pm5.08{\times}10^{0}$ & \cellcolor{gray!15}$\mathbf{1.31{\times}10^{1}\pm5.07{\times}10^{0}}$ & $1.35{\times}10^{1}\pm5.15{\times}10^{0}$ & $6.03{\times}10^{1}\pm8.68{\times}10^{0}$ \\
5.0 & $8.93{\times}10^{1}\pm1.50{\times}10^{1}$ & \cellcolor{gray!15}$\mathbf{8.92{\times}10^{1}\pm1.50{\times}10^{1}}$ & $8.96{\times}10^{1}\pm1.48{\times}10^{1}$ & $1.04{\times}10^{2}\pm4.57{\times}10^{0}$ \\

\hline
\multicolumn{5}{c}{$\theta=\frac{3\pi}{4}$} \\
\hline
0.1 & $1.10{\times}10^{-3}\pm1.30{\times}10^{-3}$ & $1.10{\times}10^{-3}\pm1.30{\times}10^{-3}$ & $2.80{\times}10^{-3}\pm2.80{\times}10^{-3}$ & \cellcolor{gray!15}$\mathbf{9.00{\times}10^{-4}\pm1.20{\times}10^{-3}}$ \\
0.2 & $4.60{\times}10^{-3}\pm5.50{\times}10^{-3}$ & $4.20{\times}10^{-3}\pm5.40{\times}10^{-3}$ & $7.20{\times}10^{-3}\pm8.20{\times}10^{-3}$ & \cellcolor{gray!15}$\mathbf{3.90{\times}10^{-3}\pm6.10{\times}10^{-3}}$ \\
0.5 & $3.95{\times}10^{-2}\pm4.84{\times}10^{-2}$ & \cellcolor{gray!15}$\mathbf{3.68{\times}10^{-2}\pm4.50{\times}10^{-2}}$ & $4.86{\times}10^{-2}\pm5.75{\times}10^{-2}$ & $7.08{\times}10^{-2}\pm1.23{\times}10^{-1}$ \\
1.0 & $1.69{\times}10^{-1}\pm1.25{\times}10^{-1}$ & \cellcolor{gray!15}$\mathbf{1.59{\times}10^{-1}\pm1.23{\times}10^{-1}}$ & $1.90{\times}10^{-1}\pm1.36{\times}10^{-1}$ & $8.11{\times}10^{-1}\pm7.21{\times}10^{-1}$ \\
2.0 & $1.83{\times}10^{0}\pm1.20{\times}10^{0}$ & \cellcolor{gray!15}$\mathbf{1.80{\times}10^{0}\pm1.19{\times}10^{0}}$ & $1.91{\times}10^{0}\pm1.23{\times}10^{0}$ & $1.16{\times}10^{1}\pm2.61{\times}10^{0}$ \\
5.0 & $2.61{\times}10^{1}\pm8.40{\times}10^{0}$ & \cellcolor{gray!15}$\mathbf{2.60{\times}10^{1}\pm8.41{\times}10^{0}}$ & $2.62{\times}10^{1}\pm8.26{\times}10^{0}$ & $3.33{\times}10^{1}\pm3.28{\times}10^{0}$ \\

\hline
\multicolumn{5}{c}{$\theta=\frac{5\pi}{6}$} \\
\hline
0.1 & $6.00{\times}10^{-4}\pm6.00{\times}10^{-4}$ & $6.00{\times}10^{-4}\pm6.00{\times}10^{-4}$ & $1.20{\times}10^{-3}\pm1.20{\times}10^{-3}$ & \cellcolor{gray!15}$\mathbf{4.00{\times}10^{-4}\pm6.00{\times}10^{-4}}$ \\
0.2 & $2.70{\times}10^{-3}\pm2.60{\times}10^{-3}$ & $2.40{\times}10^{-3}\pm2.30{\times}10^{-3}$ & $3.70{\times}10^{-3}\pm3.80{\times}10^{-3}$ & \cellcolor{gray!15}$\mathbf{2.20{\times}10^{-3}\pm2.90{\times}10^{-3}}$ \\
0.5 & $1.92{\times}10^{-2}\pm1.80{\times}10^{-2}$ & \cellcolor{gray!15}$\mathbf{1.74{\times}10^{-2}\pm1.68{\times}10^{-2}}$ & $2.20{\times}10^{-2}\pm2.05{\times}10^{-2}$ & $2.38{\times}10^{-2}\pm4.10{\times}10^{-2}$ \\
1.0 & $1.01{\times}10^{-1}\pm1.22{\times}10^{-1}$ & \cellcolor{gray!15}$\mathbf{9.32{\times}10^{-2}\pm1.19{\times}10^{-1}}$ & $1.11{\times}10^{-1}\pm1.31{\times}10^{-1}$ & $2.66{\times}10^{-1}\pm2.59{\times}10^{-1}$ \\
2.0 & $5.63{\times}10^{-1}\pm5.47{\times}10^{-1}$ & \cellcolor{gray!15}$\mathbf{5.42{\times}10^{-1}\pm5.48{\times}10^{-1}}$ & $5.88{\times}10^{-1}\pm5.64{\times}10^{-1}$ & $3.09{\times}10^{0}\pm1.33{\times}10^{0}$ \\
5.0 & $8.24{\times}10^{0}\pm4.71{\times}10^{0}$ & \cellcolor{gray!15}$\mathbf{8.06{\times}10^{0}\pm4.63{\times}10^{0}}$ & $8.28{\times}10^{0}\pm4.64{\times}10^{0}$ & $1.12{\times}10^{1}\pm1.49{\times}10^{0}$ \\
\hline
\end{tabular}
}

\caption{Mean $\pm$ standard deviation of MSE across all regimes. Bold values with gray background indicate the best (lowest mean MSE) method for each noise level.}
\label{tab:full_results}
\end{table}

\end{document}